%% file: 00_main.tex
\documentclass[preprint,12pt]{elsarticle}

\usepackage{amssymb}
\usepackage{amsmath}

\usepackage{cleveref}
\usepackage{xcolor}
\usepackage{soul}
\usepackage{multirow}
\usepackage{url}
\usepackage{makecell}

\usepackage{amsthm}
\usepackage{comment}
\theoremstyle{definition}
\newtheorem{definition}{Definition}[section]
\usepackage{booktabs}
\usepackage{array}
\usepackage{float}
\newtheorem{theorem}{Theorem}
\journal{Information Systems}

\begin{document}

\begin{frontmatter}



\title{Automated decision-making for dynamic task assignment at scale}


%% Authors and affiliations
\author[label1,label2]{Riccardo Lo Bianco\corref{cor1}}
\ead{r.lo.bianco@tue.nl}
\author[label1,label2]{Willem van Jaarsveld}
\author[label1]{Jeroen Middelhuis}
\author[label1,label2]{Luca Begnardi}
\author[label1,label2]{Remco Dijkman}

\affiliation[label1]{organization={Eindhoven University of Technology},
            addressline={5612 AZ Eindhoven},
            country={Netherlands}}

\affiliation[label2]{organization={Eindhoven Artificial Intelligence Systems Institute},
            addressline={5612 AZ Eindhoven},
            country={Netherlands}}

\cortext[cor1]{Corresponding author.}

%% Abstract
\begin{abstract}
The Dynamic Task Assignment Problem (DTAP) concerns matching resources to tasks in real time while minimizing some objectives like resource costs or task cycle time. In this work, we consider a DTAP variant where every task is a case composed of a stochastic sequence of activities. The DTAP, in this case, involves the decision of which employee to assign to which activity to process requests as quickly as possible.
In recent years, Deep Reinforcement Learning (DRL) has emerged as a promising tool for tackling this DTAP variant, but most research is limited to solving small-scale, synthetic problems, neglecting the challenges posed by real-world use cases.
To bridge this gap, this work proposes a DRL-based Decision Support System (DSS) for real-world scale DTAPS.
To this end, we introduce a DRL agent with two novel elements: a graph structure for observations and actions that can effectively represent any DTAP and a reward function that is provably equivalent to the objective of minimizing the average cycle time of tasks. The combination of these two novelties allows the agent to learn effective and generalizable assignment policies for real-world scale DTAPs.
The proposed DSS is evaluated on five DTAP instances whose parameters are extracted from real-world logs through process mining. The experimental evaluation shows how the proposed DRL agent matches or outperforms the best baseline in all DTAP instances and generalizes on different time horizons and across instances.
\end{abstract}

%% Keywords

\end{frontmatter}



\input{01_introduction}
\input{02_related_work}

\input{03_problem_description}
\input{04_method}

\input{05_experimental_setting}

\input{06_results}
\input{07_conclusions}

\bibliographystyle{elsarticle-num} 
\bibliography{references, datasets, new_dss, new_review_3}

\end{document}

%% file: 01_introduction.tex
\section{Introduction}\label{sec:introduction}
The problem of matching resources to tasks effectively, generally referred to as the Task Assignment Problem (TAP)~\cite{Kuhn1955TheProblem}, has been studied from different angles in the DSS literature.
\begin{comment}
Several variants of the TAP have been proposed, along with specialized DSS architectures, for applications spanning from spatial crowdsourcing~\cite{Wu2023}, to maintenance planning~\cite{Deng2021}, to fraud detection optimization~\cite{Vanderschueren2024}.
Solving algorithms for the classic TAP rely on the a-priori knowledge of all tasks and resources.
\end{comment}
In the classic TAP, all tasks and resources are known at the initialization and can thus be tackled by mathematical programming.
In contrast, in the Dynamic Task Assignment Problem (DTAP), tasks and resources are only known at runtime. In this work, we examine a DTAP variant that is common in administrative and industrial processes where a single final decision or product can be obtained only by completing a sequence of activities, namely, a case. 
\begin{comment}
In its most general formulation, the number, types, and order of activities in a given case are unknown when it enters the system.
\end{comment}

As an example, let us consider the case of a financial institution that receives loan applications (i.e., cases) and has to decide whether to approve or reject them. Multiple activities must be performed to produce an approval decision. For example, the possible flows of activities that lead to a decision for a single case are represented in the BPMN notation in~\cref{fig:bpmn}, in which the arrows from the left \textit{Start} activity circle to the right \textit{End} activity circle show the sequence of activities and the \textit{X} gateways represent probabilistic (exclusive) choices on the next activity to perform. At any point in time, a case is characterized by a single current activity. Employees (i.e. resources) in the company must be assigned to work on the current activity of a case by a planner. Only a subset of all available employees is active at any time to be assigned to cases according to a calendar. Once assigned, after a certain amount of time, which is based on the resource proficiency on the case's current activity, the assignment terminates, the resource is again considered active, and the case either moves to a new current activity or terminates. In this work, we consider the objective of minimizing the average cycle time of cases.

\begin{figure}[H]
    \centering
    \centerline{\includegraphics[width=\textwidth]{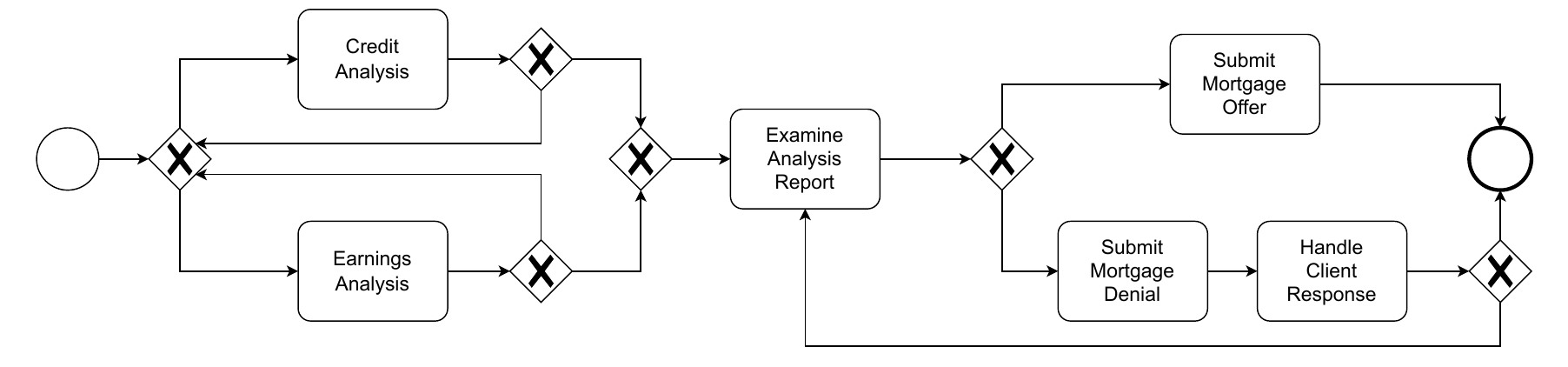}}
    \caption{A BPMN representation of a loan application process.}
    \label{fig:bpmn}
\end{figure}

In recent years, researchers have provided promising results by applying Deep Reinforcement Learning (DRL) to the DTAP~\cite{HUANG2011127, 10.1007/978-3-031-41620-0_13}, but only a few works deal with the DTAP variant considered here~\cite{scenarios, 10.1007/978-3-031-70396-6_10}. Crucially, most of these works show only proof-of-concept applications on small-scale artificial problems, and large-scale applications are not yet fully attained. 

This work aims to fill this gap by proposing a Decision Support System (DSS) that helps planners assign resources to tasks in large-scale DTAPs under realistic assumptions. To achieve this goal, according to the DRL paradigm, we propose a system architecture built on two modules: an \textit{environment} and an \textit{agent}. The former models the characteristics and dynamics of a specific DTAP instance and is parametrized through statistical analysis of real-world event logs. The latter models the decision-maker who interacts with the environment by choosing which assignments to perform. We propose a DRL agent that learns effective assignment policies thanks to two novel elements, namely, a graph-based observation inspired by~\cite{10.1007/978-3-031-70418-5_12}, which enables learning generalizable policies for the DTAP using the Proximal Policy Optimization (PPO) algorithm paired with Graph Neural Networks (GNNs), and a reward function that allows optimizing objectives that can be evaluated only as the end result of cases, such as the total cost of cases or their average cycle time. In particular, we provide the formal proof that such a reward function translates directly to the cycle time of cases, thus ensuring that the maximization of the rewards corresponds to the problem objective. These two novel elements represent the main contribution of this paper.

To validate the proposed method, we train the DRL agent in five environments, parametrized by applying process mining to real-world event logs. The results show that the agent not only matches or exceeds the performances of the best baseline heuristic but also displays remarkable generalization capabilities on time horizons longer than the training one and across different datasets.

\begin{comment}
The contributions of this work are thus threefold: a modular DSS for DTAP that separates the environment from the agent, a graph-based structure for observations and actions that allows learning policies that are reusable across different DTAP instances, and a specialized reward function that translates directly to the objective of minimizing average cycle time for cases.
\end{comment}

Against this background,~\cref{sec:related_work} is dedicated to a review of related research work. In~\cref{sec:problem}, we present the DTAP variant considered in this work by describing the environment module. In ~\cref{sec:method}, we introduce the agent module, focusing on observations, GNN structures, and rewards. In~\cref{sec:experimental_setting}, we describe the logs and the data mining technique used to parameterize the environment, as well as the performance of the agent in various settings. Finally,~\cref{sec:final} discusses the merits and limitations of the proposed approach and defines the next research steps.

%% file: 02_related_work.tex
\section{Related Work}\label{sec:related_work}

This section is dedicated to an analysis of existing work related to the problem of efficiently matching resources to tasks in the context of DSSs, as well as studies from other fields that have a direct link to the methods discussed in this paper. 

In~\cref{tab:literature_review_methods}, we present an overview of existing studies pertaining to DSSs for Task Assignment Problems (TAP). For each work, we report the TAP problem discussed (static or dynamic), whether proficiency levels (PL) are modeled, whether resource calendars (RC) are modeled, whether tasks are composed of sequences of activities (AS), and an indication of the size of validation data.

\begin{comment}
\begin{table}[h!]
\centering
\begin{tabular}{>{\centering\arraybackslash}p{2cm} >{\centering\arraybackslash}p{2cm} >{\centering\arraybackslash}p{3cm} >{\centering\arraybackslash}p{2cm} >{\centering\arraybackslash}p{2cm}}
\toprule
\textbf{Category} & \textbf{Publication} & \textbf{Method} & \textbf{Problem} & \textbf{Real-world} \\
\midrule
\multirow{2}{*}{\makecell{Disaster\\Management}} & \cite{10.1007/978-3-642-31075-1_21} & Queuing theory & DTAP & Yes \\
 & \cite{Jana2022} & Heuristics & TAP & No \\
\midrule

\multirow{2}{*}{\makecell{\\Spatial\\ Crowdsourcing}} & \cite{WU2018107} & Supervised learning & DTAP & Yes \\
 & \cite{Wu2023} & Heuristics & DTAP & Yes \\
\midrule
Routing & \cite{Gamst2024} & Metaheuristics & TAP & Yes \\
\midrule
\multirow{4}{*}{\makecell{\\\\Business\\Process\\Optimization}} 
 &
 \cite{7604008} & Heuristics & TAP & Yes \\
 & \cite{Vanderschueren2024} & Supervised learning & DTAP & Yes \\
 & Ours & Deep reinforcement learning & DTAP & Yes \\
\bottomrule
\end{tabular}
\caption{Literature relevant to DSSs for resource allocation}
\label{tab:literature_review_methods}
\end{table}
\end{comment}
\begin{table}[htbp] \centering \begin{tabular}{>{\centering\arraybackslash}p{2.5cm} >{\centering\arraybackslash}p{2cm} >{\centering\arraybackslash}p{1.5cm} >{\centering\arraybackslash}p{1.5cm} >{\centering\arraybackslash}p{1.5cm} >{\centering\arraybackslash}p{2cm}} \toprule \textbf{Study} & \textbf{Type} & \textbf{PL} & \textbf{RC} & \textbf{AS} & \textbf{Data Size} \\ \midrule \cite{10.1007/978-3-642-31075-1_21} & DTAP & No & No & No & Small \\ \cite{Jana2022} & TAP & No & No & No & Small \\ \cite{WU2018107} & DTAP & No & No & No & Large \\ \cite{Wu2023} & DTAP & No & No & No & Large \\ \cite{Gamst2024} & TAP & Yes & No & No & Large \\ \cite{7604008} & TAP & Yes & No & No & Large \\ \cite{Vanderschueren2024} & DTAP & Yes & No & No & Large \\ This paper & DTAP & Yes & Yes & Yes & Large \\ \bottomrule \end{tabular} \caption{Literature relevant to DSSs for resource allocation} \label{tab:literature_review_methods} \end{table}

In \cite{10.1007/978-3-642-31075-1_21}, the authors propose a DSS that employs queueing theory to optimize resource allocation for earthquake search and rescue operations, simulating earthquake scenarios to determine the best assignment of rescue groups to operational areas. This work is close to ours because it solves a DTAP but relies on assumptions that make the problem approachable with queuing theory.
In \cite{Jana2022}, the authors introduce a DSS that hierarchically allocates resources and tasks for disaster management. The proposed method is not validated on real-world data, and its applicability to large-scale instances is not discussed.
\cite{WU2018107} proposes a framework for a time-prediction-based task assignment approach in spatial crowdsourcing, combining supervised learning predictions with assignment heuristics. The proposed method is specific to spatial crowdsourcing and cannot be applied to other DTAP variants.
Similarly, \cite{Wu2023} refines the concepts presented in \cite{WU2018107}, but the proposed methodology is not applicable to other DTAP variants.
\cite{Gamst2024} presents a DSS for the technician routing and scheduling problem, a TAP variant, optimizing technician routes and schedules based on qualifications and time constraints. It uses a metaheuristic approach with column generation to build optimal solutions. However, the proposed method is not applicable in the dynamic setting.
In \cite{7604008}, the authors propose a heuristic-based DSS for an IT service manager to assign tasks to the most suitable IT analysts. However, the proposed method is limited to the static TAP, and no discussion regarding its application to the dynamic setting is provided.
In \cite{Vanderschueren2024}, the authors propose a method to optimally allocate limited resources to uncertain tasks by framing the DTAP as a linear assignment problem and using learning-to-rank to maximize expected profit. The model assumes stochastic worker capacity and uncertain task outcomes, similar to ours, but it does not include resource calendars, and it does not consider cases composed of sequences of activities but rather atomic tasks.

As emerges from~\cref{tab:literature_review_methods}, the works in the DSS literature have, until now, neglected problems with multiple activities per task and resource calendar.
However, in works pertaining to other domains, we can find some early attempts in this direction. In~\cite{Huang2011}, the authors model a DTAP with cases as a Markov Decision Process (MDP) to solve it through reinforcement learning. Although the evaluation is rigorous and based on real-world logs, the proposed approach is limited to small problem instances due to the use of vanilla reinforcement learning techniques. In this work, we remove this limitation through Deep Reinforcement Learning (DRL), which combines deep learning with reinforcement learning.

The use of DRL to learn assignment policies for DTAPs with tasks composed of different activities has been pioneered in~\cite{scenarios} on small-scale archetypal problems. A single work~\cite{10.1007/978-3-031-70396-6_10} tackled a large-scale DTAP with proficiency levels, resource calendars, and activity sequences. The main limitation of~\cite{10.1007/978-3-031-70396-6_10} lies in the fact that the environment used for the evaluation is parameterized by mining the processing and waiting time of the activities as a single parameter, leading to overly optimistic resource calendars. In fact, early experiments showed that the DRL approach proposed in~\cite{10.1007/978-3-031-70396-6_10} cannot learn effective policies for the DTAPs considered in this work. Compared to~\cite{10.1007/978-3-031-70396-6_10}, this paper proposes a reward function that translates directly to the problem objective and a DRL method capable of generalization across different DTAP instances.

\begin{comment}
    
This work relies heavily on GNNs to learn suitable assignment policies. The main inspiration for the GNNs proposed in this work comes from~\cite{10.1007/978-3-031-70418-5_12}, where the authors employed GNNs to solve small-scaled DTAPs. In this work, we will use similar architectures and show how they are applicable to realistic use cases.
\end{comment}

%% file: 03_problem_description.tex
\section{Problem Description}\label{sec:problem}
This section is dedicated to the description of the DTAP variant considered in this work (called, from now on, simply DTAP). To this end, we provide the formal definition of the problem, the problem state, and the transitions that trigger the changes of state.

In the DTAP, resources are assigned to cases that enter the system dynamically. The cases are composed of a sequence of activities. At any point, a single activity (the \textit{current activity}) characterizes a case. The sequence of activities that must be performed to complete the case is unknown when the case enters the system. 
Instead, each activity is associated with the likelihood of the next activity that must be performed or the completion of the case (given by the occurrence of an \textit{End} activity) by means of a probability function. 
The duration of each assignment (i.e., the time it takes for a given resource to complete a given activity) is sampled from a random variable. Moreover, the available resources are known, but only a subset of them are active at any point in time, according to a probabilistic calendar.
%The business process underlying the problem is represented in Petri Net notation in~\cref{fig:pn}.

\begin{comment}
\begin{figure}[H]
    \centering
    \centerline{\includegraphics[width=\textwidth]{images/petri_net_assigment.pdf}}
    \caption{Petri Net representation of the proposed DTAP variant}
    \label{fig:pn}
\end{figure}
\end{comment}

\begin{comment}
In~\cref{fig:dtap_example}, we propose a schematic representation of the DSS for the aforementioned DTAP variant.

\begin{figure}[H]
    \centering
    \centerline{\includegraphics[width=0.8\textwidth]{images/DTAP_example.pdf}}
    \caption{A visual representation of the proposed DSS for DTAP}
    \label{fig:dtap_example}
\end{figure}
\end{comment}

To provide a formal definition of the DTAP, we dedicate~\cref{subsec:problem_instance} to the description of the \textit{problem instance}, which is the set of immutable parameters characterizing the problem, and~\cref{subsec:problem_state} to the description of the \textit{problem state}, which is the set of mutable parameters whose values change during execution. In~\cref{subsec:transitions}, we define the behavior in terms of the \textit{transitions} that trigger changes in the problem state.

\subsection{Problem instance}\label{subsec:problem_instance}

A DTAP instance is described by the set of constants in~\cref{def:constants}.

\begin{sloppypar}
\begin{definition}[Dynamic Task Assignment Problem]\label{def:constants}
A DTAP is a tuple \((M, L, f_{\text{comp}}, G, f_{\text{trans}}, P, V, \lambda , \tau_{\text{max}})\), where:

\begin{itemize}
\item \(M\) is the set of resources available in the system.
\item \(L\) is a the of activity labels. At any time, a case \(c\) is characterized by a single activity label \(l \in L\). The set \(l\) contains at least a \textit{Start} and an \textit{End} activity, and every new case entering the system is identified by the \textit{Start} activity.
\item \(f_{\text{comp}}: G \to \mathbb{R}^{+} \times \mathbb{R}^{+} \) is a function that associates every tuple \((l, m) \in G\) with a mean \(\mu_{(l, m)}\) and a variance \(\sigma_{(l,m)}\) for the completion time. Every time a resource \(m\) is assigned to an activity label \(l\), the completion time is sampled from a Gaussian random variable with mean \(\mu_{(l, m)}\) and standard deviation \(\sigma_{(l,m)}\). To avoid negative completion times, the absolute value is considered.
\item \(G\) is a set of tuples of the form \( (l, m)\), indicating which resources \(m \in M\) can be assigned to cases whose current activity is \(l \in L\). We refer to the set of tuples that contain a given activity label \(l\) as the \textit{resource pool} of \(l\).
\item \(f_{\text{trans}}: L \to Q^{|L|}\) is a function that associates each activity label \(l\) with a probability for every activity label in \(L\). Every time an activity \(l\) is completed, a new activity label is produced for the current case by sampling from \(f_{\text{trans}}(l)\). At least one activity \(l\) must exist with a non-0 probability to produce an \textit{End} next activity.
\item \(P \in \mathbb{N}^{K}\) is a vector of \(K\) natural numbers representing the hourly resource calendar. For example, we can consider an hourly calendar with a weekly periodicity setting \(K=168\). For each hour of the week, \(P\) reports the expected number of active resources.
\item \(V \in \mathbb{N}^{|M|}\) is a vector of \(|M|\) natural numbers specifying a weight for every resource in \(M\), such that for the resource at position $i$ in the vector, the probability that this resource works compared to the probability that other resources work is $V_i$ divided by $\sum_{i=1}^{|V|} V_i$.  The weights are used to sample the resources that enter or exit the system every hour according to the values in \(C\).
\item \(\lambda \in \mathbb{R}^{+}\) is the arrival rate of cases.
\item \(\tau_{\text{max}} \in \mathbb{R}^{+}\) is the time horizon of the problem, expressed in hours.

\end{itemize}
\end{definition}
\end{sloppypar}

\subsection{Problem State}\label{subsec:problem_state}

The state of a problem instance at (simulated) time \(\tau\) is defined as follows.
\begin{comment}
The agent uses past and present data to inform decisions, while future data guides the evolution of the environment.
\end{comment}

\begin{sloppypar}
\begin{definition}[DTAP State]\label{def:state}
At a given point in time \(\tau\), a DTAP \((M, L, f_{\text{comp}}, G, f_{\text{trans}}, P, V, \lambda , \tau_{\text{max}})\) is in state \(s_{\tau} = (M^{a}_{\tau}, M^{b}_{\tau}, C^{a}_{\tau}, C^{b}_{\tau}, B_{\tau}, t)\), where:

\begin{itemize}

\item \(M^{a}_{\tau} \subseteq M\) is a set of active (i.e. available for assignment) resources.
\item \(M^{b}_{\tau} \subseteq M\) is a set of busy resources (that is, working on a case).
\item \(C^{a}_{\tau}\) is a multiset of active cases, each characterized by its current activity label \(l \in L\).
\item \(C^{b}_{\tau}\) is a multiset of busy cases, each characterized by its current activity label \(l \in L\).

\item \(B_{\tau}\) is a set of assignments, represented as tuples \( (m, c, \tau_{\text{comp}})\), indicating which resource \(m \in R^{b}_{\tau}\) is working on which case \(c \in C^{b}_{\tau}\), and the (unobservable) completion time of the assignment \(\tau\).
\item \(t\) is the current decision step, measured as the number of individual assignments of a resource to a case made until \(\tau\).
\end{itemize}

\end{definition}
\end{sloppypar}

%Do we want to have the "extra" variables here? I think we can go without...

\subsection{Problem Behavior}\label{subsec:transitions}

The behavior of each problem instance is defined as transitions from one problem state to the next, triggered by the occurrence of transitions. In this work, we consider the following transitions.

\begin{sloppypar}
\begin{definition}[DTAP Transitions]\label{def:events}
The transitions that produce changes in the DTAP state are:

\begin{enumerate}
\item \textbf{Case Arrival}: new case arrivals happen according to the interarrival rate \(\lambda\). The new cases are elements of \(C_{\tau}^{a}\) characterized by the \textit{Start} activity label. A \textit{Case Arrival} transition that produces a new case \(c\) changes the current state \(s_{\tau}\) whose set of active cases is \(C^{a}_{\tau}\) to a new state \(s_{\tau'}\) whose set of active cases is \(C^{a}_{\tau '} = C^{a}_{\tau} \cup \{ c \}\).

\textit{Plan Activity} transition occurs if there are possible assignments, according to the resource pools in \(G\). When the transition happens, the agent is called to select a tuple \( (m, l)\) from the possible assignments, corresponding to the tuple \( (m, c) \) where \(c\) is a case in the system whose current activity is \(l\). This changes the current state \(s_{\tau}\) to a new state \(s_{\tau'}\) where \(M_{\tau'}^{a} = M_{\tau}^{a} \setminus \{ m \}\), \(M_{\tau'}^{b} = M_{\tau}^{b} \cup \{ m \}\), \(C_{\tau'}^{a} = C_{\tau}^{a} \setminus \{ c \}\), and \(C_{\tau'}^{b} = C_{\tau}^{b} \cup \{ c \}\).

\item \textbf{Start Activity}: After the agent selects an available assignment during a \textit{Plan Activity} transition, the completion time for the assignment is sampled from \(f_{\text{comp}}\), generating a new element in \(B_{\tau}\). This changes the current state \(s_{\tau}\) to a new state \(s_{\tau'}\) where \(B_{\tau'} = B_{\tau} \cup \{ (m, c, \tau_{\text{comp}}) \}\). Note that, in this case, the simulated time does not progress, i.e., \(\tau = \tau '\).
    
\item \textbf{Complete Activity}: After the completion time of an element of \(B_{\tau}\) has elapsed, the resource is removed from \(M_{\tau}^{b}\) and either placed in \(M_{\tau}^{a}\) or removed from the system, according to the schedule \(P\). A new current activity \(l\) is sampled from \(f_{\text{trans}}\), and the case is either put into \(C_{\tau}^{a}\) or removed from the system if the \textit{End} activity is reached. This changes the current state \(s_{\tau}\) to a new state \(s_{\tau'}\) where \(M_{\tau'}^{b} = M_{\tau}^{b} \setminus \{ m \}\), \(M_{\tau'}^{a} = M_{\tau}^{a} \cup \{ m \}\) (if the resource is not removed), and \(C_{\tau'}^{b} = C_{\tau}^{b} \setminus \{ c \}\), with \(C_{\tau'}^{a} = C_{\tau}^{a} \cup \{ c \}\) if the case is not completed.

\item \textbf{Schedule Resources}: Every chosen time unit (in our case, every hour), the available resources in the simulation are re-evaluated based on \(P\) and \(V\). If too many resources are active, some resources in \(M_{\tau}^{a}\) are sampled to be removed (if too few resources are available in \(M_{\tau}^{a}\), resources in \(M_{\tau}^{b}\) are removed once they complete their assignment), while if too few are present, some resources in \(M \setminus (M_{\tau}^{a} \cup M_{\tau}^{b})\) are sampled and made active to reach the expected amount. This changes the current state \(s_{\tau}\) to a new state \(s_{\tau'}\) where \(M_{\tau'}^{a}\) is adjusted according to the sampling.

\item \textbf{Complete Case}: According to the probabilities that guide the \textit{Complete Activity} transition, if the next activity in the case is an end activity, the corresponding case is removed from \(C_{\tau}^{b}\). This changes the current state \(s_{\tau}\) to a new state \(s_{\tau'}\) where \(C_{\tau'}^{b} = C_{\tau}^{b} \setminus \{ c \}\).
\end{enumerate}
\end{definition}
\end{sloppypar}

\begin{comment}
The workflow diagram in~\cref{fig:system_description} illustrates how the occurrence of a transition produces new transitions (edges incoming to transitions indicate that a new transition of the pointed type will occur at some point in the future) and what state variables they affect (tuple below the transition name).

\begin{figure}[H]
    \centering
    \centerline{\includegraphics[width=\textwidth]{images/FEL.pdf}}
    \caption{A workflow representation of the effect of every transition type on the system.}
    \label{fig:system_description}
\end{figure}

It is important to note that the only transition that does not advance the simulation time \(\tau\) is \textit{Plan Activity} since its effect is to call the agent to decide on an assignment and create a \textit{Start Activity} to apply the agent's decision to the environment. The end of the simulation is not reported, as it can be triggered by any other transition if \(\tau > \tau_{\text{max}}\).
\end{comment}

%% file: 04_method.tex
\section{A DSS for Solving DTAP at Scale}\label{sec:method}

In this section, we describe a DRL agent that is suitable for solving DTAP problems at scale. In particular, we introduce two novel elements: (1) a graph-based structure for observations and actions that enables the representation of large-scale DTAPs and can easily be mapped to an input for a suitable Graph Neural Network (GNN); and (2) a reward function that allows the agent to learn effective policies to optimize objectives specified in terms of the end result of cases, focusing in particular on the minimization of average cycle time of cases.
A simple component model of the proposed DSS is presented in~\cref{fig:component_model}, distinguishing the environment and the agent. The \textit{DTAP} component represents a DTAP instance as introduced in~\cref{def:constants} and executes the transitions as in~\cref{def:events}. When assignments are possible, the \textit{DTAP} component presents the current state, as introduced in~\cref{def:state}, to the \textit{Observation Mapping} and the \textit{Reward Mapping} components. The former produces a representation of the state (i.e., observation) that is suitable as an input for the agent, while the latter produces a reward signal for the previous assignment. Based on the received observation, the \textit{Agent} module produces a numeric value, the action, that is passed to the \textit{Action Mapping}. This last module translates the received action to the corresponding assignment and passes it to the DTAP. 

\begin{figure}[htbp]
    \centering
    \centerline{\includegraphics[width=\textwidth]{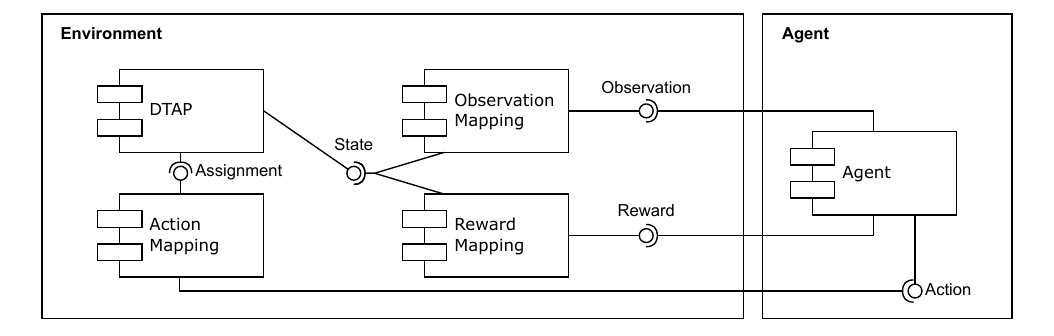}}
    \caption{A component model of the proposed DSS.}
    \label{fig:component_model}
\end{figure}

Against this background, we dedicate~\cref{subsec:DTAP_MDP} to the mapping of elements of~\cref{fig:component_model} to the MDP underlying the DRL agent. In~\cref{subsec:obs_act} we provide the description of the graph-based encoding of observations produced by the \textit{Observation Mapping} module and the graph-based encoding of the actions produced by the \textit{Agent} module. In~\cref{subsec:obs_act}, we describe the GNNs that are used to produce the actions given the observations. Finally,~\cref{subsec:rewards} is dedicated to the description of the~\textit{Reward Mapping} module.

\subsection{DTAP as a Markov Decision Process}\label{subsec:DTAP_MDP}

Online DRL algorithms learn by interaction with an environment. Such an environment is defined mathematically as a Markov Decision Process (MDP). The proposed solution method is based on elements of the MDP definition, reported in~\cref{def:mdp} as an adaptation of~\cite{Sutton1998}.

\begin{sloppypar}
\begin{definition}[Markov Decision Process]\label{def:mdp}
A Markov Decision Process (MDP) is defined by the tuple \( (O, A, \mathcal{P}, R, \gamma) \), where:

\begin{itemize}
    \item \( O \) is a set of observations.
    \item \( A \) is a set of actions.
    \item \( \mathcal{P}(o'|o, a) \) is the state transition probability, representing the probability of transitioning to observation \( o' \) from observation \( o \) after taking action \( a \).
    \item \( R(o, a) \) is the reward function, representing the immediate reward received after taking action \( a \) in observation \( o \).
    \item \( \gamma \) is the discount factor, where \( 0 \leq \gamma \leq 1 \), representing the importance of future rewards.
\end{itemize}
\end{definition}
\end{sloppypar}

In~\cref{def:mdp}, the classic MDP definition is slightly modified to avoid confusion: in particular, we refer to MDP observations instead of states, since in this work, the term \textit{state} is reserved for the problem states described in~\cref{def:state}. It is important to underline the fact that, in the MDP formulation, the time \(\tau\) is not considered. Instead, the number of actions \(t\) taken until a given observation, commonly referred to as \textit{decision step}, is tracked. For this reason, \(t\) is part of the state in~\cref{def:state}.

\begin{comment}
All the problems studied in this work have the same objective, namely the minimization of the cycle time of cases. In the context of RL, crafting reward functions that take time into account is a notoriously difficult task~\cite{pardo2022timelimitsreinforcementlearning}. In fact, time is not a dimension considered explicitly in the standard MDP formulation, where actions only change the environment's state and generate rewards. To handle problems where time must be considered as a continuous variable, it is common to rely on the Semi-Markov Decision Process (SMDP) formulation~\cite{NIPS1994_07871915}. The main problem with SMDPs is that most state-of-the-art DRL algorithms, including the one used in this work (PPO), apply only to standard MDPs. For this reason, we consider the standard MDP definition and include all transitions that do not require actions, along with the simulation time of occurrence \(\tau\), as part of the next observation.
\end{comment}

In~\cref{fig:rl_cycle_extended}, we provide a graphical representation of the proposed DSS by enriching the classic reinforcement learning cycle~\cite{Sutton1998}. The DTAP produces new states as a result of transitions. When assignments are possible in a state \(s_{\tau}\), an observation \(o_t\) is presented to the solving algorithm, where \(t\) is the decision step. The structure of \(o_t\), which is detailed in the next subsection, is a representation of the current state \(s_{\tau}\). The agent decides on a single assignment that is applied to the environment, producing a new state \(s_{\tau'}\). If other assignments are available, the simulation time is not modified, and a new observation \(o_{t+1}\) is immediately produced. If, instead, no other assignments are possible, the environment continues to produce new states until new assignments are possible.

\begin{figure}[htbp]
    \centering
    \centerline{\includegraphics[width=\textwidth]{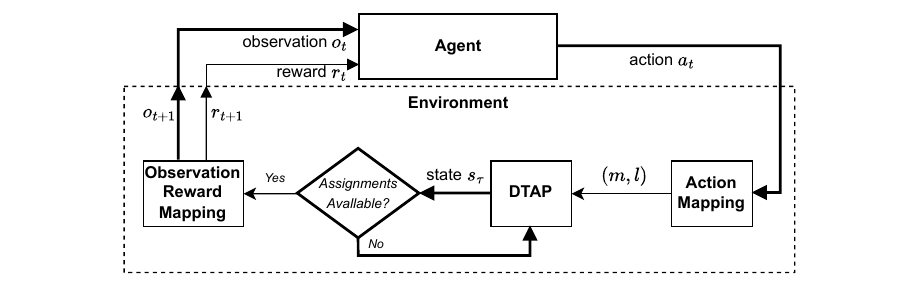}}
    \caption{MDP interaction between agent and environment for DTAP.}
    \label{fig:rl_cycle_extended}
\end{figure}

Against this background, we dedicate one subsection to each of the three main features of the proposed DRL agent: a graph-like observation, an architecture for the value and policy networks that can work seamlessly with such observations, and a specialized reward structure that mirrors the problem objective.

\subsection{Observations and Actions}\label{subsec:obs_act}

Assignment problems are commonly expressed as bipartite graph matching problems. In bipartite graph matching problems, a heterogeneous graph with two types of nodes represents the objects to be put into relations. In the case of the DTAP, we shall consider resources on one end and activity labels on the other.
An example of such a representation is given in~\cref{fig:bipartite_graph}.

\begin{figure}[htbp]
    \centering
    \centerline{\includegraphics[width=\textwidth]{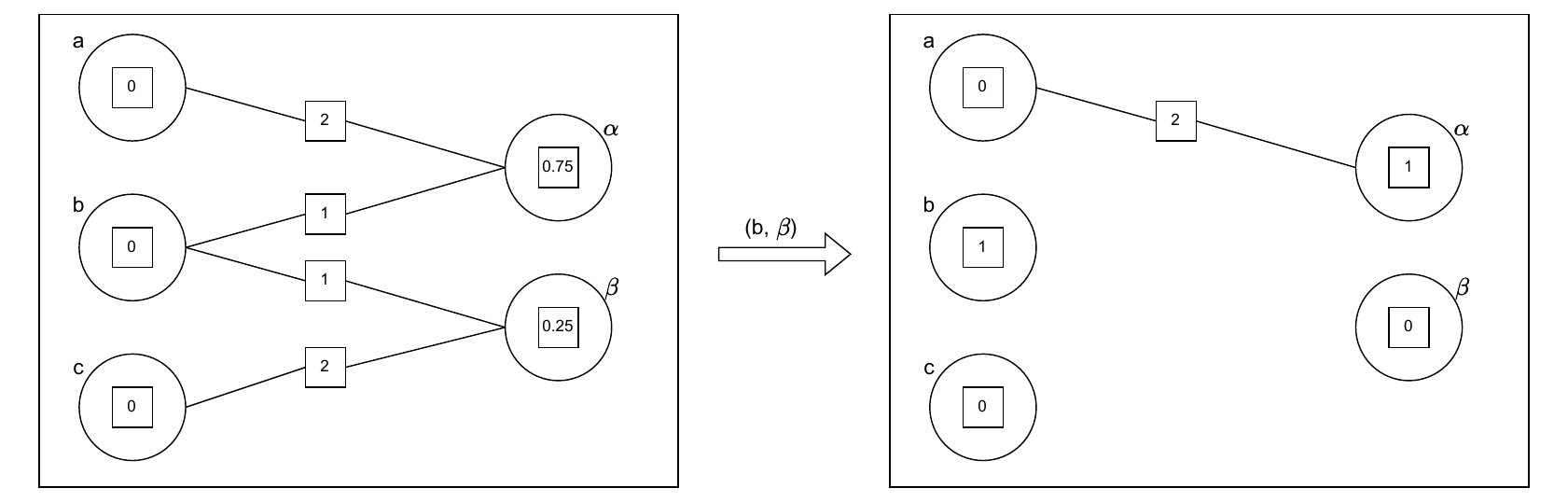}}
    \caption{A visual representation of the DTAP observation as a bipartite graph.}
    \label{fig:bipartite_graph}
\end{figure}

In~\cref{fig:bipartite_graph}, we represent an observation (left) where three resources are present, \(a\), \(b\), and \(c\), and two activity labels \(\alpha\) and \(\beta\). Resource nodes are characterized by a Boolean value that indicates whether they are active or busy (used only for masking purposes). Activity nodes are characterized by a single feature that represents the percentage of active cases whose current activity label is the one represented by the node. We do not encode the number of cases for each current activity label, but rather their ratio, in order to enforce a degree of symmetry in the representation, thus shrinking the size of the observation space. Active resources that fall into the resource pool of a given activity will be connected to the corresponding activity node via an (undirected) edge. The edge features encode the average completion time of each activity label for each resource in its resource pool. Only resource \(b\) falls into the resource pools of both activities, with an average completion time of \(1\) for both, while \(a\) can only be assigned to \(\alpha\) with an average completion time of \(2\), and \(c\) can only be assigned to \(\beta\) with an average completion time of \(2\). Based on the figure, we know that the ratio between cases whose current activity label is \(\alpha\) and cases whose current activity is \(\beta\) is \(3/4\).
Given this observation, an assignment would be performed by selecting an edge, represented as the concatenation of the resource and activity nodes that it connects. In~\cref{fig:bipartite_graph}, the agent chooses to assign resource \(b\) to (the only case having) activity label \(\beta\), leading to a new observation (right) where only cases characterized by activity label \(\alpha\) are present. When multiple instances of a given activity are present (i.e., multiple elements in \(C^{a}_{\tau}\) have the same value), we assume that a FIFO policy is adopted, where cases that entered the system first are also served first.

The proposed logic gives the foundation for creating observations that can be used to feed a GNN. However, classic GNN architectures that support heterogeneous graphs, such as the Heterogeneous graph Attention Network (HAN)~\cite{10.1145/3308558.3313562} used in this work, are incapable of effectively performing actions in the form of edge selections.
This is because, in HAN, the activity label nodes are updated with all the features of connected resources, and all the resources would be updated with the features of all connected activity label nodes, regardless of the edge features. This would make all edges connected to the nodes with the same features equal, making the representation unusable for decision-making.
To avoid this issue, we modify the bipartite graph representation by introducing a third node type, the \textit{assignment node}. Similarly to~\cite{10.1007/978-3-031-70418-5_12}, this approach casts the DTAP as a node selection problem. Assignment nodes exhibit a directed edge incoming from a single resource node, one incoming from a single activity label node, and a self-connecting edge. They are characterized by a single feature that represents the average completion time of the assignment. The directed edges ensure that the assignment nodes' embeddings are calculated without updating information in the resource and activity label nodes, and the self-connecting edges ensure that the average completion time is taken into account in the assignment node embedding. This version of the observation, which we call the \textit{assignment graph}, is reported in~\cref{fig:assignment_graph}.

\begin{figure}[htbp]
    \centering
    \centerline{\includegraphics[width=\textwidth]{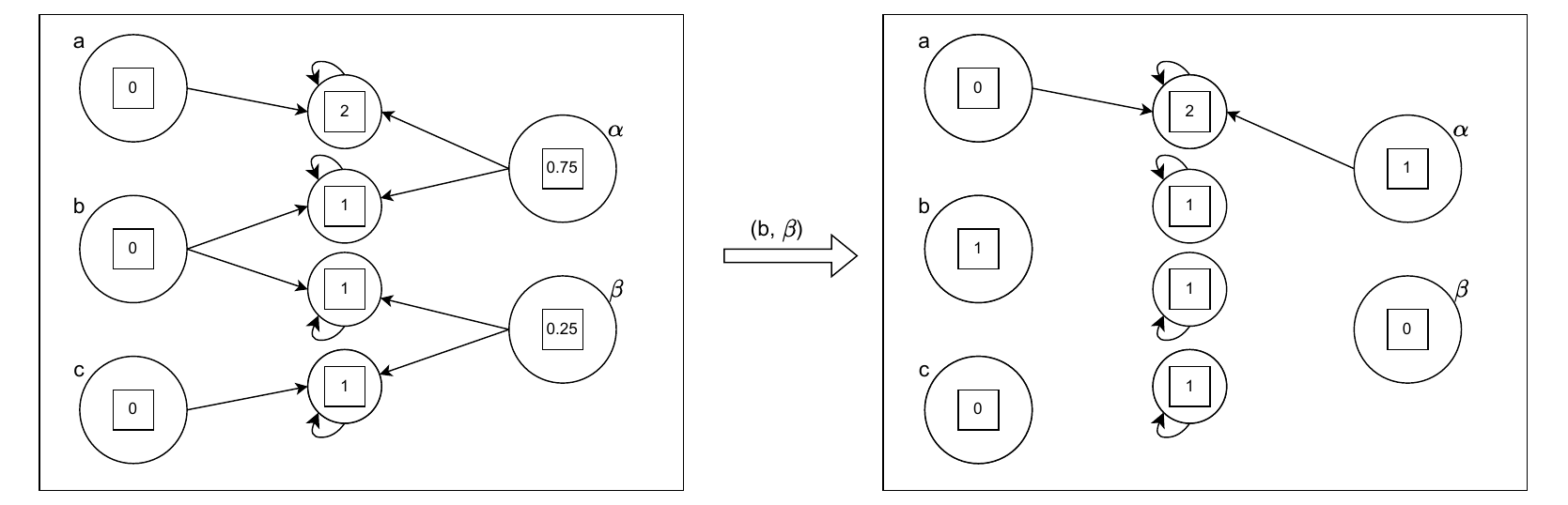}}
    \caption{A visual representation of the DTAP observation as an assignment graph.}
    \label{fig:assignment_graph}
\end{figure}

In this representation, actions are encoded as the selection of assignment nodes, and unfeasible assignments are represented as isolated nodes (and excluded from the available actions using a mask applied to the output of the policy network).

In~\cref{def:observation}, we formally define the assignment graph observation. We resort to the notation \(f(x)\) to indicate the features of element \(x\).

\begin{sloppypar}
\begin{definition}[DTAP Observation]\label{def:observation}
Given a DTAP\\ \((M, L, f_{\text{comp}}, G, f_{\text{trans}}, P, V, \lambda , \tau_{\text{max}})\) in state \(s_{\tau} = (M^{a}_{\tau}, M^{b}_{\tau}, C^{a}_{\tau}, C^{b}_{\tau}, B_{\tau}, t)\) the observation at decision step $t$ is defined as the heterogeneous graph \(o_t = (H_t, Q_t, J, Z, D_t)\) , where:

\begin{itemize}
\item \(H_t\) is the set of all resource nodes, i.e. it is created s.t. there is a one-to-one mapping $H_t\rightarrow M$.  For convenience, we refer to the resource node \(h\) associated with resource \(m\) with the notation \(h_m\). Nodes in \(H_t\) are characterized by a single boolean feature s.t. \(\forall m \in M \colon f(h_m) = 0 \text{ if } m \in M^{a}_{\tau} \text{ else } f(h_m) = 1\).
\item \(Q_t\) is the set of activity (label) nodes,  i.e. it is created s.t. there is a one-to-one mapping $Q_t\rightarrow L$. Such nodes are characterized by a single feature representing, for each element \(l \in L\), the ratio of cases \(c \in C^{a}_{\tau} | f(c) = l\).  For convenience, we refer to the activity label node \(q\) associated with activity label \(l\) with the notation \(q_l\). Since the features of the nodes in \(Q_t\) represent a ratio, it is always the case that \(\sum_{q \in Q_t}{f(q)} = 1\).
\item \(J\) is a fixed set of assignment nodes,  i.e. it is created s.t. there is a one-to-one mapping \(J\rightarrow G\). For convenience, we refer to the assignment node $j \in J$ corresponding to \((l,m) \in G\) with the notation \(j_{(l,m)}\). Nodes in \(J\) are characterized by a single feature representing the average completion time \(\mu_{(l, m)}\) of the connected resource \(m\) when assigned to an activity of the connected activity node \(l\), i.e. \(\forall (l, m) \in G, j_{(l,m)} \in J \text{ such that } f(j_{(l,m)}) = \mu_{(l, m)}\).

\item \(Z = \{(j,j)|j \in J\}\) is a fixed set of self-connecting edges on each assignment node.

\item \(D_t\) is a set of edges \((h, j) \) from (available) resource nodes to assignment nodes and edges \((q, j)\) from (available) activity label nodes to assignment nodes. Edges incoming to an assignment node \(j_{(l, m)}\) exist if the corresponding resource \(m\) is available for assignment, i.e. if \(f(h_r) = 0\), and if there is at least one case whose current activity label is \(l\), i.e. \(\forall j_{(l, m)} \in J \colon \exists (h_m, j_{(l, m)}) \in D_t \land \exists (q_l, j_{(l, m)} \in D_t) \text{ iff } f(h_m) = 0 \land f(q_l) > 0\).

\end{itemize}
\end{definition}
\end{sloppypar}

To speed up the training process, observations are passed to the agent only when there are at least two possible assignments, and completion times during training are forced to be the means \(\mu _{(l,m)}\) of \(f_{\text{comp}}\). In addition, the features of nodes of the same type are standardized by subtracting their mean and dividing by the standard deviation to ensure training stability.

\subsection{Actor and Critic Neural Networks}
Having described the characteristics of the observations, we can define the components of the neural networks at the core of the DRL agent.

In DRL applications, the agent follows a policy \(\pi\), which is a strategy that defines the action it will take given an observation to maximize cumulative rewards. The policy can be interpreted as a function \(\pi: O \times A \to [0, 1] \), associating every action in an observation with the probability of taking that action, or as a function \(\pi: \mathcal{O} \to \mathcal{A}\), associating every observation to the single most valuable action.
Most DRL algorithms aim at maximizing either the action-value function \( Q^\pi(o, a) \):
\begin{equation}\label{action_value_function}
Q^\pi(o, a) = \mathbb{E} \left[ \sum_{t=0}^{\infty} \gamma^t R(o_t, a_t) \mid o_0 = o, a_0 = a, \pi \right]
\end{equation}

or the value function:

\begin{equation}\label{value_function}
V^\pi(o) = \mathbb{E} \left[ \sum_{t=0}^{\infty} \gamma^t R(o_t, a_t) \mid o_0 = o, \pi \right]
\end{equation}

where \(\mathbb{E}\) represents the expected value, while \(R\) and \(\gamma\) are, respectively, the reward function and the discount factor introduced in~\cref{def:mdp}.

PPO aims at maximizing a weighted combination of the two functions, thus adopting two networks: the policy network approximates \(Q^\pi(o, a)\), while the value network approximates \(V^\pi(o)\).
In~\cref{fig:actor_critic}, the main elements of the two NNs employed in this work are reported.

\begin{figure}[htbp]
    \centering
    \centerline{\includegraphics[width=\textwidth]{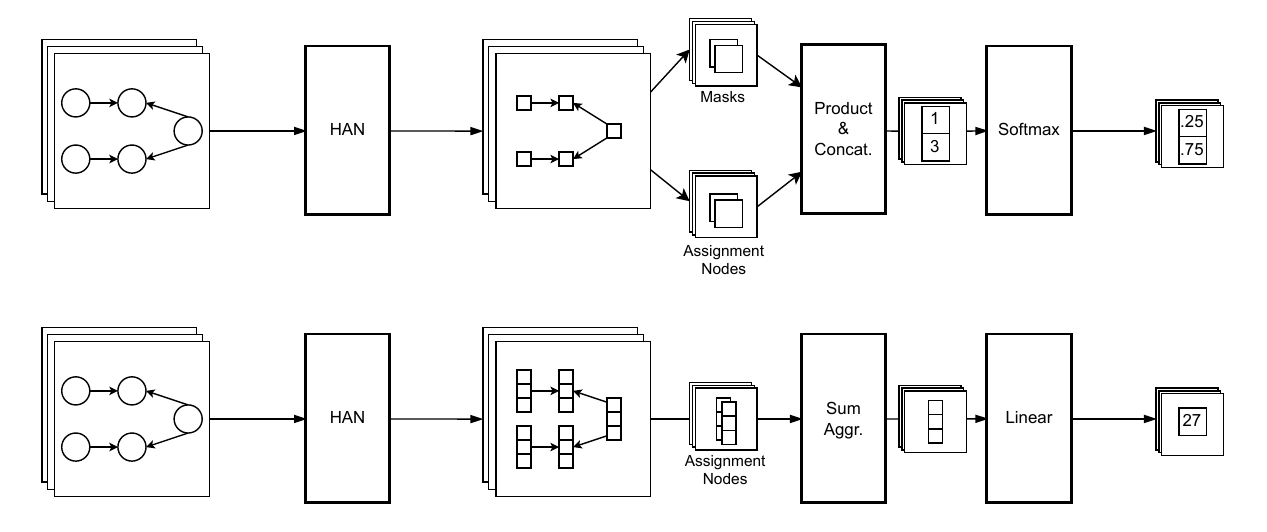}}
    \caption{A visual representation of the policy network (top) and the value network (bottom).}
    \label{fig:actor_critic}
\end{figure}

The two networks are basic encoder-decoder architectures, where the HAN layer acts as an encoder in both cases.
In the policy network, the HAN layer outputs a single embedded value for each node. The embedded assignment nodes are then directly summed to a mask vector containing \(0\) for selectable nodes and \(-\infty\) for non-selectable ones. The results are then concatenated and passed to a logarithmic softmax layer to extract the probability of selecting each assignment node.
In the value network, the HAN layer embeds each node as a vector (in the experiments, node embeddings have \(16\) elements). The embeddings of assignment nodes are then summed point-wise, and the resulting vector is used as input to a linear layer that outputs a single value.

\subsection{Rewards}\label{subsec:rewards}

Crafting reward functions that take into account the time dimension is a notoriously difficult task~\cite{pardo2022timelimitsreinforcementlearning}. In our case, since the objective is to minimize the total cycle time of cases, a logical reward function would return the (negative) cycle time every time a case is completed. However, in early experiments, it was found that this reward function is ineffective in the DTAP considered in this work. This is in line with our understanding of ineffective reward functions, since rewards produced with this method: 1) are sparse since most assignments do not produce rewards; and 2) are delayed since the effect of an assignment is only seen a long time after the action is taken. These properties tend to prevent the learning of effective policies based on the reward function. 

To overcome this issue, we propose an alternative reward function and demonstrate that the sum of rewards obtained at the end of an episode using the alternative method equals the sum of cycle times. In other words, we show that the alternative reward function is directly related to the objective of minimizing the cycle time of cases. Moreover, we will provide empirical evidence that the proposed reward function enables us to learn effective policies for our problem. 

The proposed reward function tracks the (negative) number of cases in the system between transitions, multiplied by the simulation time between transitions. The sum of these values is returned as a reward every time an action (assignment) is performed. The queue of active and busy cases at the end of an episode is truncated so that the residual cycle time cumulated at the end of an episode for cases that are not completed is summed to the total cycle time.
It is easy to see that this equals the sum of cycle times by observing that the sum of rewards over an episode equals the area under the curve of the function of the number of cases in the system over time. 

\begin{figure}[htbp]
    \centering
    \centerline{\includegraphics[width=\textwidth]{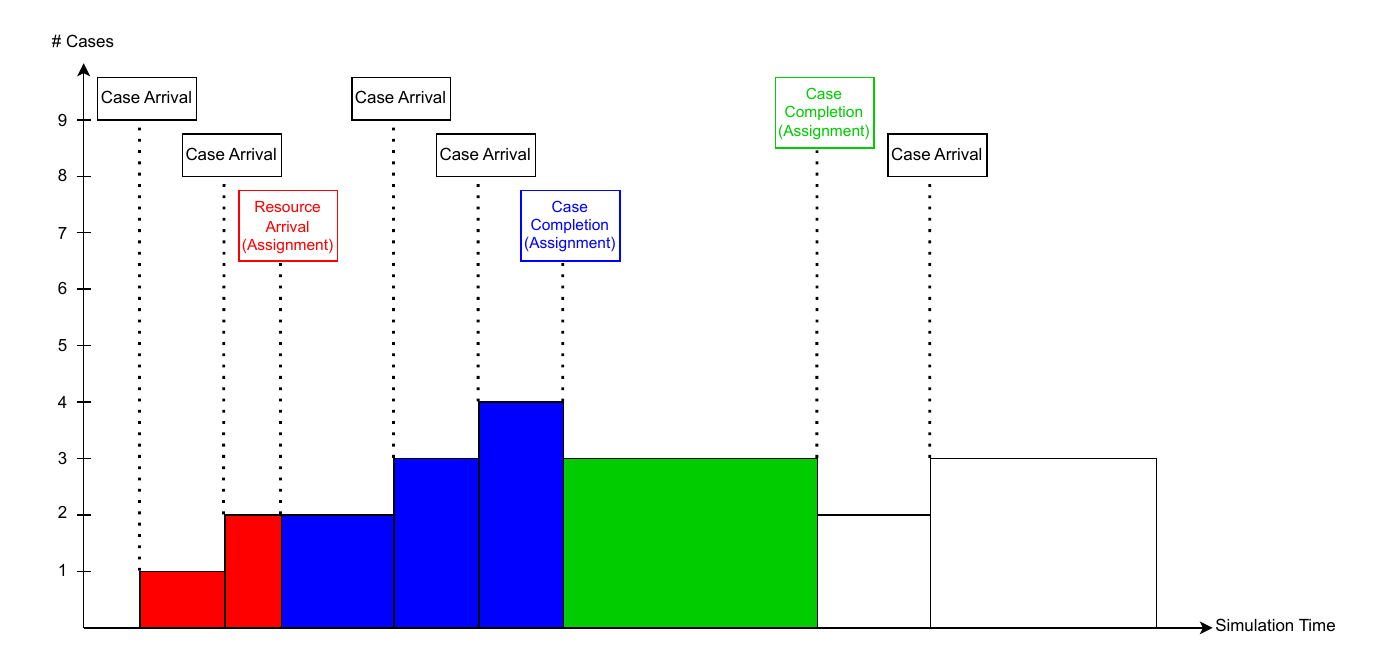}}
    \caption{A visual representation of the reward function.}
    \label{fig:reward_function}
\end{figure}

In Fig.~\ref{fig:reward_function}, we report the occurrence of three assignment actions (respectively, in red, blue, and green) and the area under the curve (in the same color) that represents their reward. The horizontal axis shows the simulation time, and the vertical axis shows the number of active cases in the system. Two cases arrive in the system and consequently raise the number of active cases until a resource arrives, making it possible to execute the first assignment action. This action receives a reward that is represented by the red area under the curve. Subsequently, two more cases arrive until a case is completed, freeing a resource and thus enabling a new assignment action, producing a reward equal to the blue area. Lastly, a new case completion allows for a new assignment whose reward is represented by the green area.

We now formally define the reward function and prove that the sum of these rewards in an episode is equal to the negative sum of the cycle times. Consequently, the reward function incentivizes the agent to minimize the average number of cases in the system, resulting in the minimization of the cycle time of cases.

\begin{definition}[Reward Function that Minimizes Sum of Cycle Times]
Let $t\in \mathbb{N}$ be the sequence number of a decision step. Furthermore, let $\mathrm{T}_t = \{ \tau_0, \cdots, \tau_n \}$ be the multiset of all times a state transition happened between decision step $t-1$ and $t$ or since the start of the simulation if $t=0$ and let $C_t = \{ 
 |C^{a}_{\tau_0}| + |C^{b}_{\tau_0}|, \cdots, |C^{a}_{\tau_n}| + |C^{b}_{\tau_n}| \} $ be the multiset of the number of active and busy cases in those states.

We define the reward of decision $t$ as:

\begin{equation}\label{eq:reward_func}
r_t = -\sum_{i=0}^{| \mathrm{T} _t \setminus \{ \tau_n \} |}{c_{i}(\tau_{i+1} - \tau_{i})}
\end{equation}    
\end{definition}

\begin{theorem}\label{proof:reward}
Consider a DTAP with a limited time horizon \(\tau_{\text{max}}\) and constant arrival rate of cases \(\lambda\). Let \(t_{\text{max}}\) be the sequence number of the last decision taken before $\tau_{\text{max}}$.

Assuming the reward function proposed in~\cref{eq:reward_func} and a discount factor \( \gamma = 1 \), the sum of rewards over an episode \(R_{t_{\text{max}}}\) is equal to the (negative) product between the total number of cases that entered the system, \(C_{\text{max}} = |C^{a}_{\tau_{\text{max}}}| + |C^{b}_{\tau_{\text{max}}}| + |C^{c}_{\tau_{\text{max}}}|\), and the total simulation time and the average cycle time of cases \(W\):

\begin{equation}\label{eq:sum_rewards}
R_{t_{\text{max}}} = \sum_{t=0}^{t_{\text{max}}} r_t = - C_{\text{max}}W
\end{equation}

\end{theorem}

\begin{proof}

\begin{comment}
Little's Law~\cite{Little1961AW} states that the average number of cases in the system $N$ is equal to the average arrival rate of cases $\lambda$ times the average time a case spends in the system (i.e., the cycle time) $W$, as reported in~\cref{little_law}:

\begin{equation}\label{little_law}
N = \lambda W
\end{equation}

In a reinforcement learning problem with a fixed horizon \(t_{\text{max}}\), the agent learns a policy that maximizes the discounted sum of rewards accumulated at the end of a trajectory~\cref{eq:sum_of_rewards}:

\begin{equation}\label{eq:sum_of_rewards}
R_{t_{\text{max}}} = \sum_{t=0}^{\infty} \gamma^t r_{t}
\end{equation}

Given that \( \gamma = 1 \), the sum of rewards in an episode is therefore equal to the product between the total number of cases that entered the system, \(C_{\text{max}}\), and the total simulation time, \(\tau_{\text{max}}\). Given that the arrival rate of cases $\lambda$ is constant, $T$ is equal to the product of $\lambda$ and the episode's total duration $\tau_{\text{max}}$:

\begin{equation}\label{eq:num_cases}
T = \lambda \tau_{\text{max}}
\end{equation}

We can thus reformulate the sum of rewards (divided by the total amount of cases in the episode) as follows:

\begin{equation}\label{eq:sum_rewards_proof}
\frac{\sum r}{T} = \frac{-\sum \frac{1}{\lambda W}}{\lambda \tau} = -\frac{1}{W}
\end{equation}
\end{comment}

The sum of rewards over an episode can be written as:

\begin{equation}
R_{t_{\text{max}}} = \sum_{t=0}^{t_{\text{max}}} r_t = \sum_{t=0}^{t_{\text{max}}} \left( -\sum_{i=0}^{| \mathrm{T} t \setminus { \tau_n } |}{c{i}(\tau_{i+1} - \tau_{i})} \right)
\end{equation}

We can interchange the order of summation:

\begin{equation}
R_{t_{\text{max}}} = -\sum_{t=0}^{t_{\text{max}}} \sum_{i=0}^{| \mathrm{T} t \setminus { \tau_n } |}{c{i}(\tau_{i+1} - \tau_{i})}
\end{equation}

Since \(\tau_{i+1} - \tau_{i}\) represents the time interval between consecutive decision steps, the reward accumulates the total time passed in the system by all cases.
The average cycle time \(W\) is calculated as the fraction between the total time spent by all cases and the total number of cases in a trace \(C_{\text{max}}\). Thus, the average cycle time (W) can be expressed as:
\begin{equation}
W = \frac{\sum_{t=0}^{t_{\text{max}}} \sum_{i=0}^{| \mathrm{T} t \setminus { \tau_n } |}{c{i}(\tau_{i+1} - \tau_{i})}}{C_{\text{max}}}
\end{equation}

Multiplying both sides by \(C_{\text{max}}\):

\begin{equation}
C_{\text{max}} W = \sum_{t=0}^{t_{\text{max}}} \sum_{i=0}^{| \mathrm{T} t \setminus { \tau_n } |}{c{i}(\tau_{i+1} - \tau_{i})}
\end{equation}

Therefore, the sum of rewards over the episode is:

\begin{equation}
R_{t_{\text{max}}} = - C_{\text{max}} W
\end{equation}

\end{proof}

%% file: 05_experimental_setting.tex
\section{Experimental Evaluation}\label{sec:experimental_setting}

This section is dedicated to an empirical demonstration of the main claim of this work, namely, that the method proposed in~\cref{sec:method} is suitable to solve DTAPs at scale.
In particular, we dedicate~\cref{subsec:parameters} to describing how the real-world scale DTAPs used for evaluation are parameterized, and~\cref{subsec:results} to examine the results of the trained agents in various settings.

\subsection{DTAP Parametrization}\label{subsec:parameters}

In this work, we employ process mining techniques to parameterize five large-scale DTAP problem instances from real-world event logs. Process mining involves extracting knowledge from event logs to discover, monitor, and improve real-world processes using data from the information systems that support them~\cite{b9ce1ceaafbe4a20bc39b8b6026ea9be}. 

An event log is a record of events that occur within an information system. The event logs considered in process mining typically include information such as the case ID (which identifies the case), the activity name, and the timestamp of when the activity occurred. In addition to these, the event logs used in this work also report which resource was assigned to activities that require resources to be performed. In~\cref{fig:event_log}, we propose an example of an event log header produced by the business process reported in~\cref{fig:bpmn}.

\begin{figure}[H]
    \centering
    \centerline{\includegraphics[width=0.8\textwidth]{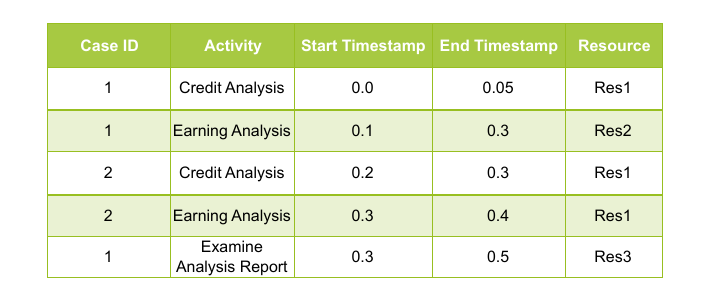}}
    \caption{An example header of an event log.}
    \label{fig:event_log}
\end{figure}

With reference to~\cref{def:constants}, the elements of \(L\) are the unique values of the \textit{Activity} column, while the elements of \(M\) are the unique values of the \textit{Resource} column. The function \(f_{comp}\) is parameterized based on the timestamps in the \textit{Timestamp} column.
Only a subset of resources can be assigned to each activity label, according to the resource pools in \(G\). Each resource in the resource pool of an activity label has different proficiency in completing that activity, expressed in \(f_{\text{comp}}\) as a Gaussian random variable whose mean and variance are estimated from the values in the \textit{Start Timestamp} and the \textit{End Timestamp} columns. A resource is included in the resource pool of an activity if the resource has been assigned to the activity at least $2$ times in the log (this way, it is always possible to calculate the completion time variance). Moreover, resources follow a probabilistic weekly working schedule that is also mined from the logs. The number of times a resource has been assigned to an activity in each hour of the week is used to parameterize the elements of \(V\) so that resources that appeared more often in the dataset during a given hour of the week are more likely to be included in \(M_{\tau}^{a}\) during that hour. The process mining technique employed in this work is similar to that described in~\cite{Camargo2020AutomatedLogs}. We extend the method to include probabilistic resource calendars as in~\cite{10271965}. Although more complex methods have been proposed in recent years~\cite{DBLP:journals/corr/abs-2103-11944}, their performance is only slightly better than those obtained through~\cite{Camargo2020AutomatedLogs} and only in a limited set of metrics, at the expense of significantly more complex implementations. Given that the main objective of this work is to have a realistic and challenging set of problems to solve for our agent, we favored the simplicity and transparency of~\cite{Camargo2020AutomatedLogs}. For the same reason, we do not discuss the details of the method or provide a formal comparison between different process mining approaches.

We apply the proposed process mining method to five publicly available datasets: \textbf{BPI2012}~\cite{bpi2012}, the log of interactions for loan applications at a Dutch financial institution; \textbf{BPI2017}~\cite{bpi2017}, a more recent version of the process in BPI2012, with different process structure, activity labels, and resources; \textbf{MICRO}, a synthetic log made available by Microsoft for training purposes\footnote{Download available at https://learn.microsoft.com/en-us/power-automate/process-mining-tutorial}; \textbf{PROD}~\cite{10.1007/978-3-031-70396-6_10}, a log of a heavy machinery production process; \textbf{CONS}~\cite{10.1007/978-3-031-70396-6_10}, a log of interactions between a university secretary's office and students.

In~\cref{table:data_mining}, we report the main characteristics of the five DTAP instances with reference to the elements of~\cref{def:constants}. The column \textbf{\(|L|\)} reports the number of unique activity labels, \textbf{\(|M|\)} the number of unique resources, \textbf{\(\lambda\)} the arrival rate of cases (in cases per hour), \textbf{\(\bar{|C|}\)} the mean and variance of the number of activities per case, \textbf{\(\bar{|G|}\)} the mean and variance of the number of resources per resource pool, and \textbf{\(\bar{P}\)} the mean and variance of the number of available resources per hour of the week (i.e., the elements of \(P\)).

\begin{table}[ht]
\centering
\begin{tabular}{>{\centering\arraybackslash}p{2cm} >{\centering\arraybackslash}p{0.5cm} >{\centering\arraybackslash}p{0.8cm} >{\centering\arraybackslash}p{0.6cm} >{\centering\arraybackslash}p{2.2cm} >{\centering\arraybackslash}p{2.2cm} >{\centering\arraybackslash}p{2.2cm}}
\toprule
 \textbf{Dataset} & \textbf{\(|L|\)} & \textbf{\(|M|\)} & \textbf{\(\lambda\)} & \textbf{\(\bar{|C|}\)} & \textbf{\(\bar{|G|}\)} & \textbf{\(\bar{P}\)} \\
\midrule
BPI2012 & 6 & 55 & 4.2 & \(2.3 \pm 1.2\) & \(33.2 \pm 15.2\) & \(1.8 \pm 1.2\) \\
BPI2017 & 7 & 145 & 4.8 & \(7.8 \pm 5.1\) & \(81.8 \pm 46.0\) & \(2.8 \pm 3.7\) \\
MICRO & 13 & 8 & 0.1 & \(7.5 \pm 3.6\) & \(4.2 \pm 1.0\) & \(3.8 \pm 3.1\) \\
CONS & 16 & 559 & 0.2 & \(5.2 \pm 2.2\) & \(27.2 \pm 32.0\) & \(9.3 \pm 4.0\) \\
PROD & 13 & 42 & 0.1 & \(19.6 \pm 19.4\) & \(5.8 \pm 3.1\) & \(12.9 \pm 7.3\) \\
\bottomrule
\end{tabular}
\caption{Overview of BPI2012, BPI2017, MICRO, PROD, and CONS properties.}
\label{table:data_mining}
\end{table}

The values of \(\lambda\) and \(G\) reported in~\cref{table:data_mining} are adjusted to ensure the resulting systems are stable under the Shortest Processing Time (SPT) policy (detailed in the next section). In~\cref{fig:regime}, we report the average number of cases in the system under the SPT policy over a 28-day horizon.

\begin{figure}[ht]
    \centering
    \centerline{\includegraphics[width=\textwidth]{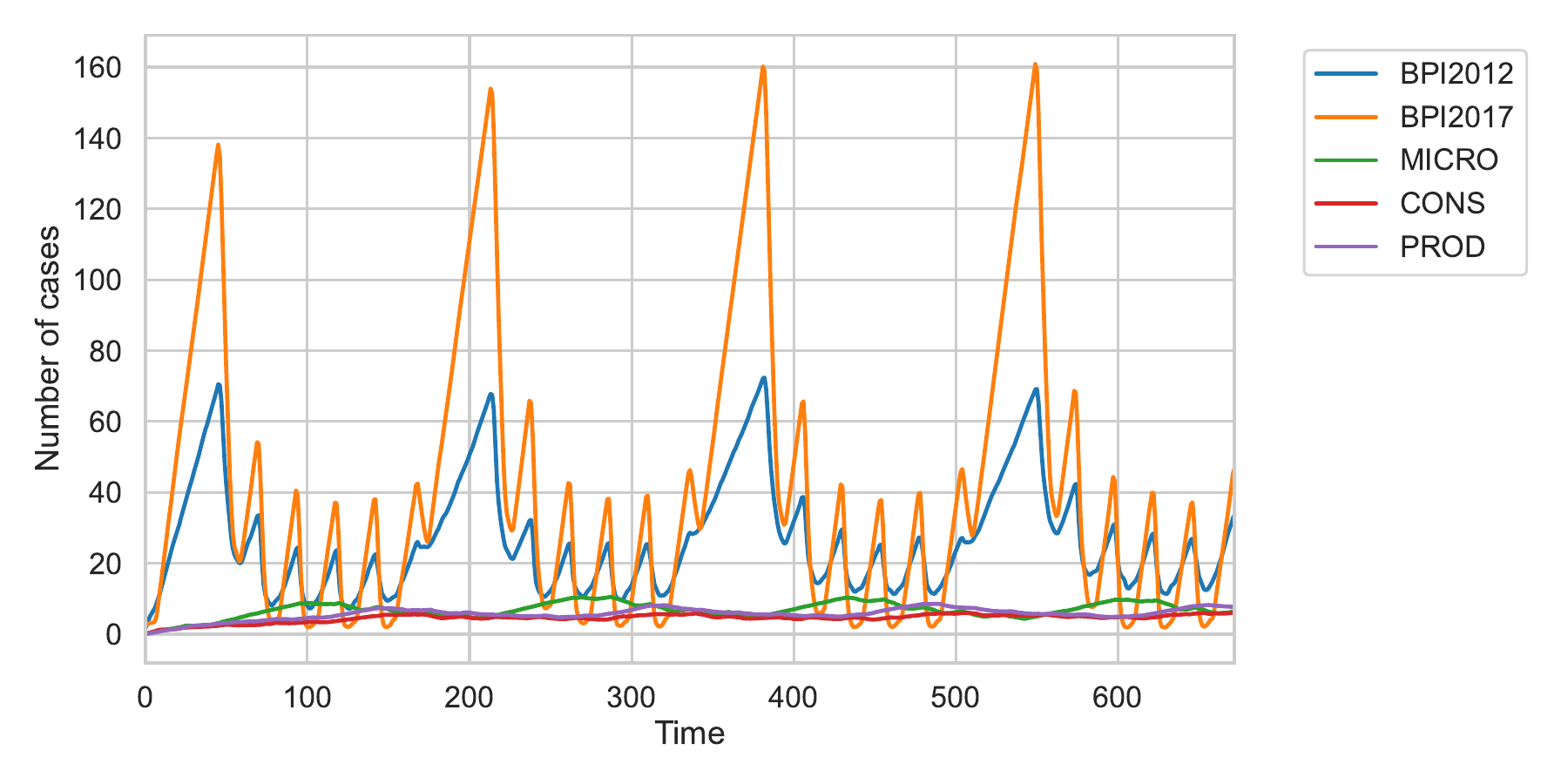}}
    \caption{The number of cases in the system when following the SPT policy over 4 weeks.}
    \label{fig:regime}
\end{figure}

%% file: 06_results.tex
\subsection{Results}\label{subsec:results}

The simulations described in \cref{sec:experimental_setting} are used to train the DRL agent on 7-day traces. The same simulations are then used to produce a set of 1000 independent 7-day traces with different policies: the random policy (Random), which selects assignments randomly; the First-In-First-Out (FIFO) policy, which prioritizes cases based on their entry time; the Shortest Processing Time (SPT) policy, which assigns activities to the resource that completes them the fastest on average: the DRL approach (PPO) proposed in this work. In~\cref{tab:results_7d}, we report the means and standard deviations of the average cycle times for every policy.

\begin{table}[ht]
\centering
\resizebox{\textwidth}{!}{%
\begin{tabular}{>{\centering\arraybackslash}m{2.1cm} >{\centering\arraybackslash}m{2.1cm} >{\centering\arraybackslash}m{2.1cm} >{\centering\arraybackslash}m{2.3cm} >{\centering\arraybackslash}m{2.3cm}}
\toprule
\multirow{2}{*}{\textbf{Dataset}} & \multicolumn{4}{c}{\textbf{Method}} \\
\cmidrule(lr){2-5}
 & \textbf{Random} & \textbf{FIFO} & \textbf{SPT} & \textbf{PPO} \\
\midrule
BPI2012 & 17.4 \(\pm\) 11.0 & 14.8 \(\pm\) 9.4 & 5.8 \(\pm\) 2.9 & \textbf{5.4 \(\pm\) 2.6} \\
BPI2017 & 9.7 \(\pm\) 2.7 & 9.8 \(\pm\) 2.5 & 7.0 \(\pm\) 2.1 & \textbf{6.3 \(\pm\) 1.3} \\
MICRO & 47.4 \(\pm\) 10.5 & 48.6 \(\pm\) 12.4 & \textbf{38.4 \(\pm\) 12.0} & \textbf{37.9 \(\pm\) 11.5} \\
CONS & 31.6 \(\pm\) 7.8 & 31.9 \(\pm\) 9.0 & 24.5 \(\pm\) 8.1 & \textbf{23.5 \(\pm\) 7.9} \\
PROD & 52.9 \(\pm\) 10.4 & 52.4 \(\pm\) 10.8 & \textbf{42.9 \(\pm\) 10.6} & \textbf{42.6 \(\pm\) 10.6} \\
\bottomrule
\end{tabular}
}
\caption{Performance comparison on 7-day traces (best result in bold text).}
\label{tab:results_7d}
\end{table}

The results on the $7$ day horizon instances clearly show that PPO matches or exceeds the performance of SPT in all instances. In particular, the differences between the average cycle times obtained by PPO and SPT in BPI2012, BPI2017, and CONS are statistically significant according to a t-test with a significance level of \(10^{-2}\). For BPI2012, BPI2017, and CONS, the reduction in average cycle time between SPT and PPO is, respectively, approximately \(6\%\), \(11\%\), and \(4\%\). In the case of MICRO and PROD, the two policies do not present statistically significant differences in terms of average cycle time.

In~\cref{fig:ppo_vs_spt}, we report the percentage of observations in which PPO and SPT take the same actions (in blue) and in which they take different actions (in orange).

\begin{figure}[ht]
    \centering
    \centerline{\includegraphics[width=0.6\textwidth]{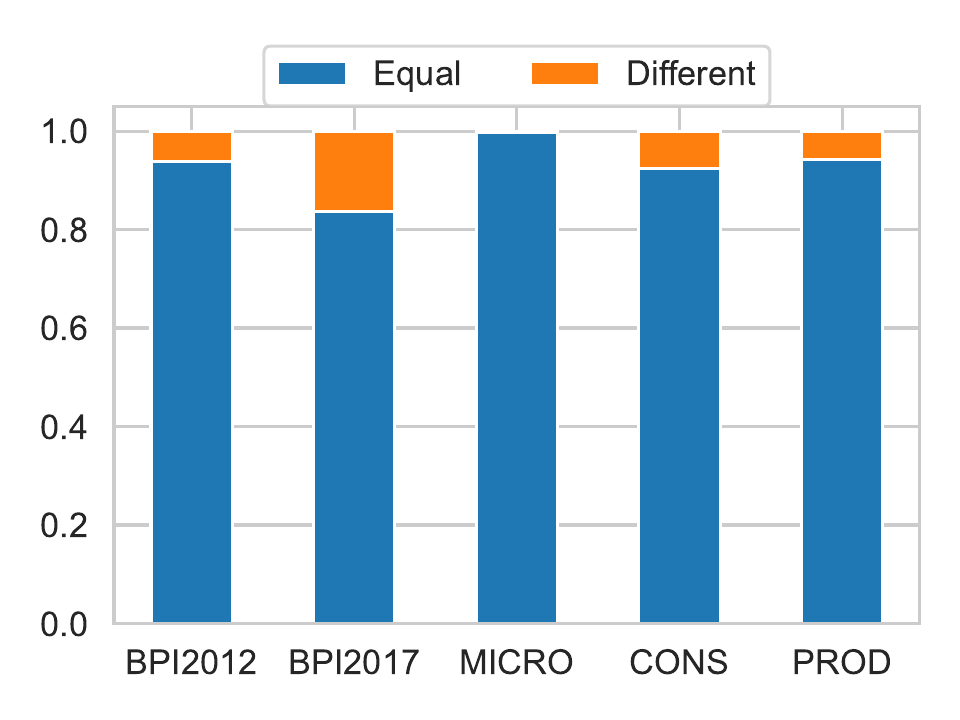}}
    \caption{Percentage of equal (blue) and different (orange) actions between PPO and SPT on a set of \(10^3\) observations (with at least two alternatives).}
    \label{fig:ppo_vs_spt}
\end{figure}

It can be observed that the DRL agent behaves differently depending on the problem. In the two extremes, in MICRO PPO mimics SPT in all the observed states, whereas in BPI2017 it behaves differently in around $18\%$ of the cases. This observation aligns with the results in~\cref{tab:results_7d}, where PPO notably outperforms SPT in BPI2017. 
This is similar to what we observe in BPI2012 and CONS, although on a smaller scale.
In the case of PROD, however, we observe that PPO behaves differently from SPT in around $6\%$ of the cases while obtaining the same results. This suggests that for the PROD simulation, different policies can be followed to obtain similar, near-optimal results. This can happen in the case of different assignments with similar average completion times.

In~\cref{table:results_28d}, we provide the average and standard deviations of mean cycle times over $10^3$ replicates of $28$ day traces. Notice that, for PPO, the models that were trained on a $7$ days horizon are used for inference.

\begin{table}[ht]
\centering
\begin{tabular}{>{\centering\arraybackslash}m{2.1cm} >{\centering\arraybackslash}m{2.1cm} >{\centering\arraybackslash}m{2.1cm} >{\centering\arraybackslash}m{2.3cm} >{\centering\arraybackslash}m{2.3cm}}
\toprule
\multirow{2}{*}{\textbf{Dataset}} & \multicolumn{4}{c}{\textbf{Method}} \\
\cmidrule(lr){2-5}
 & \textbf{Random} & \textbf{FIFO} & \textbf{SPT} & \textbf{PPO} \\
\midrule
BPI2012 & 48.6 \( \pm \) 26.1 & 29.7 \( \pm \) 16.2 & \textbf{6.4 \( \pm \) 1.9} & \textbf{6.3 \( \pm \) 1.9} \\
BPI2017 & 14.2 \( \pm \) 3.5 & 13.4 \( \pm \) 2.5 & 8.6 \( \pm \) 1.3 & \textbf{6.9 \( \pm \) 0.9} \\
MICRO & 61.2 \( \pm \) 11.6 & 66.1 \( \pm \) 15.7 & \textbf{51.0 \( \pm \) 12.5} & \textbf{49.9 \( \pm \) 10.9} \\
CONS & 42.4 \( \pm \) 7.6 & 40.9 \( \pm \) 7.9 & \textbf{33.2 \( \pm \) 7.5} & \textbf{33.0 \( \pm \) 7.7} \\
PROD & 79.5 \( \pm \) 10.9 & 79.1 \( \pm \) 11.2 & \textbf{58.6 \( \pm \) 9.1} & \textbf{58.6 \( \pm \) 9.8} \\
\bottomrule
\end{tabular}
\caption{Performance comparison on 28-day traces (best result in bold text).}
\label{table:results_28d}
\end{table}

Since the system is initially empty, the average cycle times registered on the 28-day horizon are higher than those recorded on the 7-day horizon in all cases. Interestingly, the performance gap between SPT and PPO is more pronounced in the case of BPI2017 (around $23\%$), while in the case of BPI2012 and CONS, the two policies do not show statistically significant differences according to a t-test with significance level $10^-2$. This may be explained by taking into account the percentage of actions the two policies take differently, shown in~\cref{fig:ppo_vs_spt}: in the case of BPI2017, the PPO policy behaves substantially differently from SPT, while in the case of BPI2012 and CONS this difference is less pronounced. It is possible that, for these last two datasets, the different actions have an impact only when the system is still empty, whereas their importance is negligible when the system is full. However, PPO always performs at least on par with SPT when applied to 28-day traces in all of the proposed simulations.

A remarkable property of GNNs is their ability to generalize across different domains. In the context of this work, this means that a policy learned by PPO on a given simulation can be applied seamlessly to any of the other simulations. 
Given the previous results, it is particularly interesting to observe the performance obtained testing PPO with the weights learned on BPI2017 on the other four datasets. 
In~\cref{table:results_generalized}, we show how the algorithm trained on each simulation performs on the DTAP instance used to train it (value on the diagonal) and on the other instances.

\begin{table}[ht]
\centering
\begin{tabular}{>{\centering\arraybackslash}m{1.8cm} >{\centering\arraybackslash}m{1.8cm} >{\centering\arraybackslash}m{1.8cm} >{\centering\arraybackslash}m{1.8cm} >{\centering\arraybackslash}m{1.8cm} >{\centering\arraybackslash}m{1.8cm}}
\toprule
\multirow{2}{*}{\makecell{\textbf{Test}\\\textbf{Dataset}}} & \multicolumn{5}{c}{\textbf{Train Dataset}} \\
\cmidrule(lr){2-6}
 & \textbf{BPI2012} & \textbf{BPI2017} & \textbf{MICRO} & \textbf{CONS} & \textbf{PROD} \\
\midrule
BPI2012 & \(\mathbf{5.4 \pm 2.6}\) & \(\mathbf{5.4 \pm 2.6}\) & \(\mathbf{5.5 \pm 2.6}\) & \( 10.8 \pm 7.5\) &  \(9.6 \pm 6.2\)\\
BPI2017 & \(6.9 \pm 1.9\) & \(\mathbf{6.3 \pm 1.3}\) & \( 6.9 \pm 2.0\) & \( 7.7 \pm 2.0 \)  & \(7.4 \pm 2.3\)  \\
MICRO & \(37.6 \pm 11.1\) & \(38.1 \pm 12.0\) & \(37.9 \pm 11.5\) & \( 38.2 \pm 11.9\) & \(38.7 \pm 11.7\) \\
CONS & \(23.6 \pm 8.1\) & \(23.4 \pm 8.4\) & \( 24.2 \pm 8.4\) & \( 23.5 \pm 7.9\) & \(24.0 \pm 8.2\) \\
PROD & \(42.5 \pm 10.0\) & \(42.3 \pm 10.1\) & \(42.5 \pm 10.5\) & \( 42.1 \pm 9.9\) & \(42.6 \pm 10.6\) \\
\bottomrule
\end{tabular}
\caption{Performance of the DRL model trained on each problem applied to the others on $7$ day traces (best result in bold text, omitted if the differences are not statistically significant).}
\label{table:results_generalized}
\end{table}

As emerges from~\cref{table:results_generalized}, the agent trained in BPI2017 produces results on par with the agent trained in the test environment in every instance, while none of the other agents could match the original's performance when tested on BPI2017. Interestingly, the agents trained on CONS and PROD produced suboptimal results when tested on BPI2012 and BPI2017, while the opposite is not true. There is no obvious explanation for this behavior, and a future research effort shall be devoted to exploring the possibilities and limitations of the proposed model in terms of generalization.

%% file: 07_conclusions.tex
\section{Conclusions}\label{sec:final}
This work proposes a modular DSS for a DTAP variant where a final decision or product is obtained by completing a sequence of activities. We provide a formal description of the environment and embed it in an MDP formulation that allows a DRL agent to learn an effective assignment policy. We propose a graph-based observation for the agent, along with suitable GNNs for the policy and value network. We also introduce a specialized reward function to train the agent, formally proving the equivalence between maximizing this reward and minimizing the average cycle time of cases. The evaluation is performed on five DTAP instances, whose parameters are extracted from real-world event logs. The DRL agent learns policies that perform better than or equal to the best baseline in every environment. Moreover, the DRL agent shows generalization capabilities in episodes longer than those used for training and across different DTAP instances.\footnote{Code and data are available at https://github.com/lobiaminor/bpo\_challenge\_gdrl}

This work demonstrates the fitness of DRL as a general solution method in a DSS for DTAPs. By formulating complete and separate descriptions for the two DSS modules, environment and agent, we provide a useful tool to support human decision-makers. In fact, an agent trained in a simulated environment can either be used directly to decide on the assignments to make in the real world or to give online recommendations to human decision-makers. Moreover, the environment can be mined directly from historical data, but it can also be used to study scenarios previously unseen. For example, the arrival rate of cases can be increased to study performances under heavy loads, or new activities can be introduced to study the effect of changes in the business processes underlying the DTAP.

Although the methods proposed in this work lift some of the restrictions of previous solutions, they still overlook some possible real-world challenges. Most notably, the DTAP defined in this work assumes that the environment's dynamics remain stationary during execution, which is often not the case in the real world. Additionally, the algorithm presented in this work does not account for the possibility of postponing a decision to the future, a capability that can lead to more effective assignment policies. Lastly, the cases in this work presented only sequential activities while, often, activities in a case can be carried out in parallel.
These limitations require further refinements in the environment and the agent, which will be introduced in future works.

%% file: 00_main.bbl
\begin{thebibliography}{10}
\expandafter\ifx\csname url\endcsname\relax
  \def\url#1{\texttt{#1}}\fi
\expandafter\ifx\csname urlprefix\endcsname\relax\def\urlprefix{URL }\fi
\expandafter\ifx\csname href\endcsname\relax
  \def\href#1#2{#2} \def\path#1{#1}\fi

\bibitem{Kuhn1955TheProblem}
H.~W. Kuhn, \href{https://onlinelibrary.wiley.com/doi/abs/10.1002/nav.3800020109}{{The Hungarian method for the assignment problem}}, Naval Research Logistics Quarterly 2~(1-2) (1955) 83--97.
\newblock \href {https://doi.org/https://doi.org/10.1002/nav.3800020109} {\path{doi:https://doi.org/10.1002/nav.3800020109}}.
\newline\urlprefix\url{https://onlinelibrary.wiley.com/doi/abs/10.1002/nav.3800020109}

\bibitem{HUANG2011127}
Z.~Huang, W.~{van der Aalst}, X.~Lu, H.~Duan, \href{https://www.sciencedirect.com/science/article/pii/S0169023X1000114X}{Reinforcement learning based resource allocation in business process management}, Data \& Knowledge Engineering 70~(1) (2011) 127--145.
\newblock \href {https://doi.org/https://doi.org/10.1016/j.datak.2010.09.002} {\path{doi:https://doi.org/10.1016/j.datak.2010.09.002}}.
\newline\urlprefix\url{https://www.sciencedirect.com/science/article/pii/S0169023X1000114X}

\bibitem{10.1007/978-3-031-41620-0_13}
R.~Lo~Bianco, R.~Dijkman, W.~Nuijten, W.~van Jaarsveld, Action-evolution petri nets: A framework for modeling and solving dynamic task assignment problems, in: C.~Di~Francescomarino, A.~Burattin, C.~Janiesch, S.~Sadiq (Eds.), Business Process Management, Springer Nature Switzerland, Cham, 2023, pp. 216--231.

\bibitem{scenarios}
J.~Middelhuis, R.~L. Bianco, E.~Scherzer, Z.~A. Bukhsh, I.~J. B.~F. Adan, R.~M. Dijkman, \href{https://arxiv.org/abs/2304.09970}{Learning policies for resource allocation in business processes} (2024).
\newblock \href {http://arxiv.org/abs/2304.09970} {\path{arXiv:2304.09970}}.
\newline\urlprefix\url{https://arxiv.org/abs/2304.09970}

\bibitem{10.1007/978-3-031-70396-6_10}
F.~Meneghello, J.~Middelhuis, L.~Genga, Z.~Bukhsh, M.~Ronzani, C.~Di~Francescomarino, C.~Ghidini, R.~Dijkman, Optimizing resource allocation policies in real-world business processes using hybrid process simulation and deep reinforcement learning, in: A.~Marrella, M.~Resinas, M.~Jans, M.~Rosemann (Eds.), Business Process Management, Springer Nature Switzerland, Cham, 2024, pp. 167--184.

\bibitem{10.1007/978-3-031-70418-5_12}
R.~Lo~Bianco, R.~Dijkman, W.~Nuijten, W.~van Jaarsveld, A universal approach to feature representation in dynamic task assignment problems, in: A.~Marrella, M.~Resinas, M.~Jans, M.~Rosemann (Eds.), Business Process Management Forum, Springer Nature Switzerland, Cham, 2024, pp. 197--213.

\bibitem{10.1007/978-3-642-31075-1_21}
A.~Rasekh, A.~R. Vafaeinezhad, Developing a gis based decision support system for resource allocation in earthquake search and rescue operation, in: B.~Murgante, O.~Gervasi, S.~Misra, N.~Nedjah, A.~M. A.~C. Rocha, D.~Taniar, B.~O. Apduhan (Eds.), Computational Science and Its Applications -- ICCSA 2012, Springer Berlin Heidelberg, Berlin, Heidelberg, 2012, pp. 275--285.

\bibitem{Jana2022}
S.~Jana, R.~Majumder, P.~P. Menon, D.~Ghose, Decision support system (dss) for hierarchical allocation of resources and tasks for disaster management, Operations Research Forum 3 (9 2022).
\newblock \href {https://doi.org/10.1007/s43069-022-00148-6} {\path{doi:10.1007/s43069-022-00148-6}}.

\bibitem{WU2018107}
P.~Wu, E.~W. Ngai, Y.~Wu, \href{https://www.sciencedirect.com/science/article/pii/S0167923618300587}{Toward a real-time and budget-aware task package allocation in spatial crowdsourcing}, Decision Support Systems 110 (2018) 107--117.
\newblock \href {https://doi.org/https://doi.org/10.1016/j.dss.2018.03.010} {\path{doi:https://doi.org/10.1016/j.dss.2018.03.010}}.
\newline\urlprefix\url{https://www.sciencedirect.com/science/article/pii/S0167923618300587}

\bibitem{Wu2023}
Z.~Wu, L.~Peng, C.~Xiang, Assuring quality and waiting time in real-time spatial crowdsourcing, Decision Support Systems 164 (1 2023).
\newblock \href {https://doi.org/10.1016/j.dss.2022.113869} {\path{doi:10.1016/j.dss.2022.113869}}.

\bibitem{Gamst2024}
M.~Gamst, D.~Pisinger, Decision support for the technician routing and scheduling problem, Networks 83 (2024) 169--196.
\newblock \href {https://doi.org/10.1002/net.22188} {\path{doi:10.1002/net.22188}}.

\bibitem{7604008}
Y.~Duan, C.-H. Yeh, Tads: A task assignment decision support system for it services, in: 2016 IEEE 11th Conference on Industrial Electronics and Applications (ICIEA), 2016, pp. 2472--2477.
\newblock \href {https://doi.org/10.1109/ICIEA.2016.7604008} {\path{doi:10.1109/ICIEA.2016.7604008}}.

\bibitem{Vanderschueren2024}
T.~Vanderschueren, B.~Baesens, T.~Verdonck, W.~Verbeke, A new perspective on classification: Optimally allocating limited resources to uncertain tasks, Decision Support Systems 179 (4 2024).
\newblock \href {https://doi.org/10.1016/j.dss.2023.114151} {\path{doi:10.1016/j.dss.2023.114151}}.

\bibitem{Huang2011}
Z.~Huang, W.~M. V.~D. Aalst, X.~Lu, H.~Duan, Reinforcement learning based resource allocation in business process management, Data and Knowledge Engineering 70 (2011) 127--145.
\newblock \href {https://doi.org/10.1016/j.datak.2010.09.002} {\path{doi:10.1016/j.datak.2010.09.002}}.

\bibitem{Sutton1998}
R.~S. Sutton, A.~G. Barto, Reinforcement Learning: An Introduction, 2nd Edition, The MIT Press, 2018.

\bibitem{10.1145/3308558.3313562}
X.~Wang, H.~Ji, C.~Shi, B.~Wang, Y.~Ye, P.~Cui, P.~S. Yu, \href{https://doi.org/10.1145/3308558.3313562}{Heterogeneous graph attention network}, in: The World Wide Web Conference, WWW '19, Association for Computing Machinery, New York, NY, USA, 2019, p. 2022–2032.
\newblock \href {https://doi.org/10.1145/3308558.3313562} {\path{doi:10.1145/3308558.3313562}}.
\newline\urlprefix\url{https://doi.org/10.1145/3308558.3313562}

\bibitem{pardo2022timelimitsreinforcementlearning}
F.~Pardo, A.~Tavakoli, V.~Levdik, P.~Kormushev, \href{https://arxiv.org/abs/1712.00378}{Time limits in reinforcement learning} (2022).
\newblock \href {http://arxiv.org/abs/1712.00378} {\path{arXiv:1712.00378}}.
\newline\urlprefix\url{https://arxiv.org/abs/1712.00378}

\bibitem{b9ce1ceaafbe4a20bc39b8b6026ea9be}
W.~{Van der Aalst}, Process mining: data science in action, Springer, 2016.
\newblock \href {https://doi.org/10.1007/978-3-662-49851-4} {\path{doi:10.1007/978-3-662-49851-4}}.

\bibitem{Camargo2020AutomatedLogs}
M.~Camargo, M.~Dumas, O.~Gonz{\'{a}}lez-Rojas, {Automated discovery of business process simulation models from event logs}, Decision Support Systems 134 (7 2020).
\newblock \href {https://doi.org/10.1016/j.dss.2020.113284} {\path{doi:10.1016/j.dss.2020.113284}}.

\bibitem{10271965}
O.~López-Pintado, M.~Dumas, Discovery and simulation of business processes with probabilistic resource availability calendars, in: 2023 5th International Conference on Process Mining (ICPM), 2023, pp. 1--8.
\newblock \href {https://doi.org/10.1109/ICPM60904.2023.10271965} {\path{doi:10.1109/ICPM60904.2023.10271965}}.

\bibitem{DBLP:journals/corr/abs-2103-11944}
M.~Camargo, M.~Dumas, O.~G. Rojas, \href{https://arxiv.org/abs/2103.11944}{Learning accurate business process simulation models from event logs via automated process discovery and deep learning}, CoRR abs/2103.11944 (2021).
\newblock \href {http://arxiv.org/abs/2103.11944} {\path{arXiv:2103.11944}}.
\newline\urlprefix\url{https://arxiv.org/abs/2103.11944}

\bibitem{bpi2012}
B.~van Dongen, \href{https://data.4tu.nl/articles/dataset/BPI_Challenge_2012/12689204/1}{Bpi challenge 2012} (2012).
\newblock \href {https://doi.org/10.4121/uuid:3926db30-f712-4394-aebc-75976070e91f} {\path{doi:10.4121/uuid:3926db30-f712-4394-aebc-75976070e91f}}.
\newline\urlprefix\url{https://data.4tu.nl/articles/dataset/BPI_Challenge_2012/12689204/1}

\bibitem{bpi2017}
B.~van Dongen, \href{https://data.4tu.nl/articles/dataset/BPI_Challenge_2017/12696884/1}{Bpi challenge 2017} (2017).
\newblock \href {https://doi.org/10.4121/uuid:5f3067df-f10b-45da-b98b-86ae4c7a310b} {\path{doi:10.4121/uuid:5f3067df-f10b-45da-b98b-86ae4c7a310b}}.
\newline\urlprefix\url{https://data.4tu.nl/articles/dataset/BPI_Challenge_2017/12696884/1}

\end{thebibliography}
